\def\year{2019}\relax
\documentclass[letterpaper]{article} 
\usepackage{aaai19}  
\usepackage{times}  
\usepackage{helvet}  
\usepackage{courier}  
\usepackage{url}  
\usepackage{graphicx}  

\usepackage{booktabs} 
\usepackage{amsfonts}
\usepackage{enumitem}
\usepackage[scaled]{beramono}
\usepackage{tabularx}
\usepackage{multirow}
\usepackage{subcaption}
\usepackage[T1]{fontenc}
\usepackage[font=small,labelfont=bf,tableposition=top]{caption}
\usepackage{amssymb,amsmath,amsthm}
\usepackage[linesnumbered, ruled]{algorithm2e}
\usepackage{enumitem}
\usepackage{cleveref}
\newtheorem{theorem}{Theorem}

\usepackage{color} 

\begin{document}
%
\title{Efficient and Scalable Multi-task Regression on Massive Number of Tasks}
\author{Xiao He \\
NEC Labs Europe\\
Heidelberg, Germany\\
xiao.he@neclab.eu
\And Francesco Alesiani\\
NEC Labs Europe\\
Heidelberg, Germany\\
francesco.alesiani@neclab.eu
\And Ammar Shaker\\
NEC Labs Europe\\
Heidelberg, Germany\\
ammar.shaker@neclab.eu
} 
 
\maketitle
\begin{abstract}

Many real-world large-scale regression problems can be formulated as Multi-task Learning (MTL) problems with a massive number of tasks, as in retail and transportation domains. However, existing MTL methods still fail to offer both the generalization performance and the scalability for such problems. Scaling up MTL methods to problems with a tremendous number of tasks is a big challenge. Here, we propose a novel algorithm, named \textbf{Convex Clustering Multi-Task regression Learning (CCMTL)}, which integrates with \emph{convex clustering} on the $k$-nearest neighbor graph of the prediction models. Further, CCMTL efficiently solves the underlying convex problem with a newly proposed optimization method. CCMTL is \textbf{accurate}, \textbf{efficient} to train, and empirically \textbf{scales linearly} in the number of tasks. On both synthetic and real-world datasets, the proposed CCMTL outperforms seven state-of-the-art (SoA) multi-task learning methods in terms of prediction accuracy as well as computational efficiency. On a real-world retail dataset with $23,812$ tasks, CCMTL requires only around $30$ seconds to train on a single thread, while the SoA methods need up to hours or even days.
\end{abstract}

\section{Introduction}
Multi-task learning (MTL) is a branch of machine learning that aims at exploiting the correlation among tasks. To achieve this, the learning of different tasks is performed jointly. It has been shown that learning task relationships can transfer knowledge from information-rich tasks
to information-poor tasks \cite{zhang2014regularization} so that overall generalization error can be reduced. With this characteristic, MTL 
has been successfully applied in use cases ranging from Transportation \cite{deng2017situation} to Biomedicine \cite{li2018multi}.



Various multi-task learning algorithms have been proposed in the literature, Zhang and Yang \cite{Zhang2017Survey} wrote a comprehensive survey on state-of-the-art (SoA) methods. For instance, feature learning approach \cite{argyriou2007multi} and low-rank approach \cite{ji2009accelerated,chen2011integrating} assume all the tasks are related, which may not be true in real-world applications. Task clustering approaches \cite{bakker2003task,jacob2009clustered,kumar2012learning} can deal with the situation where different tasks form clusters. However, despite being accurate, these latter methods are computationally expensive for problems with a large number of tasks. 

In real-world regression applications, the number of tasks can be tremendous. For instance, in the retail market, shops' owners would like to forecast the sales amount of all the products based on the historical sales information and external factors. For a typical shop, there are thousands of products, where each product can be modeled as a separate regression task. In addition, if we consider large retail chains, where the number of shops in a given region is in the order of hundreds,
the total number of learning tasks easily grows up to hundreds of thousands. A similar scenario can also be found in other applications, e.g. demand prediction in transportation, where each task is one public transport station demand in a city at a certain time per line. Here, at least tens of thousands of tasks are expected. 


In all these scenarios, MTL approaches that exploit relationships among tasks are appropriate. Unfortunately most of existing SoA multi-task learning methods cannot be applied, because either they are not able to cope with task heterogeneity or are too computationally expensive, scaling super-linearly in the number of tasks. 

To tackle these issues, this paper introduces a novel algorithm, named \textbf{C}onvex \textbf{C}lustering \textbf{M}ulti-\textbf{T}ask regression \textbf{L}earning (CCMTL). It integrates the objective of \emph{convex clustering} \cite{hocking2011clusterpath} into the multi-task learning framework. 
CCMTL is efficient with linear runtime in the number of tasks, yet provides accurate prediction. The detailed contributions of this paper are fourfold:

\begin{enumerate}[leftmargin=2\parindent]
	\item \textbf{Model:} A new model for multi-task regression that integrates with \emph{convex clustering} on the $k$-nearest neighbor graph of the prediction models;
	\item \textbf{Optimization Method:} A new optimization method for the proposed problem that is proved to converge to the global optimum;
	\item \textbf{Accurate Prediction:} Accurate predictions on both synthetic and real-world datasets in retail and transportation domains which outperform eight SoA multi-task learning methods;
	\item \textbf{Efficient and Scalable Implementation:} An efficient and linearly scalable implementation of the algorithm.
On a real-world retail dataset with $23,812$ tasks, the algorithm requires only around $30$ seconds to terminate, whereas SoA methods typically need up to hours or even days.
\end{enumerate}

Over the remainder of this paper, we introduce the proposed algorithm, present the mathematical proofs and a comprehensive experimental setup that benchmarks CCMTL against the SoA on multiple datasets. Lastly, we discuss related works and conclude the paper.

\section{The Proposed Method}
Let us consider a dataset with $T$ regression tasks $\{X_t, Y_t\}=\{(x_{ti}, y_{ti}): i\in\{1,...,N_t\}\}$ for each task $t\in\{1, ..., T\}$. $X_t \in \mathbb{R}^{P\times N_t}$ consists of $N_t$ samples and $P$ features, while $Y_t \in \mathbb{R}^{1\times N_t}$ is the target for task $t$. There are totally $N$ samples, where $N=\sum_{t=1}^TN_t$. We consider linear models in this paper. Let $W = (W_1^T, ..., W_T^T) \in \mathbb{R}^{T\times P}$, where $W_t\in \mathbb{R}^{P\times 1}$ represent the weight vector for task $t$. 

\subsection{Model}
Task clustering MTL methods learns the task relationships by integrating with $k$-means \cite{jacob2009clustered,zhou2011malsar} or matrix decomposition \cite{kumar2012learning}. Unfortunately, these methods are expensive to train, making them impractical for problems with a massive number of tasks.

Recently, \emph{convex clustering} \cite{hocking2011clusterpath} has attracted much attention. It solves clustering of data $X$ as a regularized reconstruction problem:
\begin{equation}
	\min_{U}\frac{1}{2}||X-U||_F^2 + \frac{\lambda}{2}\sum_{i,j \in G}||U_i - U_j||_2
	\label{eq:cc}
\end{equation}
where $U_i$ is the new representation for the $i$th sample, and $G$ is the $k$-nearest neighbor graph on $X$. It is critical to note that the edge terms involve the $\ell_2$ norm, not the squared $\ell_2$ norm, which is essential to achieve clustering. Further, it has been shown that \emph{convex clustering} is efficient with linear scalability \cite{chi2015splitting,shah2017robust,he2018robust}.

Suppose we have a noisy observation of the true prediction models $Z\in \mathbb{R}^{T\times P}$; we would like to learn less noisy models $W$ by \emph{convex clustering}: $Z$ and $W$ would take the role of $X$ and $U$ in Problem (\ref{eq:cc}), respectively. Since $Z$ is not available in practice and the target of MTL is the prediction, we use the prediction error instead of the reconstruction error and form the problem as follows:
\begin{equation}
	\min_{W}\frac{1}{2}\sum_{t=1}^T||W_t^TX_t-y_t||_2^2 + \frac{\lambda}{2}\sum_{i,j \in G}||W_i - W_j||_2
	\label{eq:obj}
\end{equation}
where $\lambda$ is a regularization parameter and $G$ is 
the $k$-nearest neighbor graph on the prediction models
learned independently for each task.

Note that if we use $\ell_1$ norm $||W_i-W_j||_1$ as the regularizer, Problem (\ref{eq:obj}) equals to the Fused Multi-task Learning (FuseMTL) \cite{zhou2012modeling,chen2010graph}.  However, it has been shown that $\ell_2$ norm works better for clustering in most cases \cite{hocking2011clusterpath}. This improvement is also confirmed for MTL in the experimental section.

\subsection{Optimization}
The Problem (\ref{eq:obj}) lies in the general framework of NetworkLasso \cite{hallac2015network}, and hence it could be solved by NetworkLasso, which is based on the general alternating direction method of multipliers (ADMM). However, NetworkLasso is not specifically designed for the multi-task regression problem. Further, in our experiments, we find its performance on single thread is rather slow and the convergence is affected by its hyperparameters. 

Chi and Lange \cite{chi2015splitting} propose an alternating minimization algorithm (AMA) for \emph{convex clustering}, which has been shown to be faster than the ADMM based method. Even if the approach can also be applied to solve Problem (\ref{eq:obj}), its convergence is only guaranteed under certain conditions. Later, Wang et al. \cite{wang2016robust} propose a variation of the AMA method for \emph{convex clustering}. We test this method to solve Problem (\ref{eq:obj}), but the convergence is not always empirically achieved. 

Shah and Koltun \cite{shah2017robust} propose an efficient method for \emph{continuous clustering} with a non-convex regularization function. Inspired by this method, we propose a new efficient algorithm for Problem (\ref{eq:obj}) by iteratively solving a structured linear system. 

First we introduce an auxiliary variable $L=\{l_{i,j}\}$ for each connection between node $i$ and $j$ in graph $G$, $l_{i,j}\geq 0$, and we form the following new problem:
\begin{equation}
	\min_{W, L}\frac{1}{2}\sum_{t=1}^T||W_t^TX_t-y_t||_2^2 + \frac{\lambda}{2}\sum_{i,j \in G}(l_{i,j}||W_i - W_j||_2^2 + \frac{1}{4}l_{i,j}^{-1})
	\label{eq:objnew}
\end{equation}

\begin{theorem}
	\label{th:lij}
	The optimal solution $W^*$ of Problem (\ref{eq:obj}) and Problem (\ref{eq:objnew}) are the same if
	\begin{equation}
		l_{i,j} = \frac{1}{2||W_i-W_j||_2}
		\label{eq:lij}
	\end{equation}
\end{theorem}
\begin{proof}
	Theorem \ref{th:lij} can be simply proved by 
    substituting Eq. (\ref{eq:lij}) into Problem (\ref{eq:objnew}), obtaining Problem (\ref{eq:obj}). 
\end{proof}

Intuitively, Problem (\ref{eq:objnew}) learns the weights $l_{i,j}$ for the squared $\ell_2$ norm regularizer $||W_i-W_j||_2^2$ as in Graph regularized Multi-task Learning (SRMTL) \cite{zhou2011malsar}. With noisy edges, squared $\ell_2$ norm forces uncorrelated tasks to be close, while $\ell_2$ norm is more robust which is confirmed for MTL in the experimental section.

To solve Problem (\ref{eq:objnew}), we optimize $W$ and $L$ alternately. When $W$ is fixed, we get the derivative of Problem (\ref{eq:objnew}) with respect to $l_{ij}$, set it to zero and get the update rule as shown in Eq. (\ref{eq:lij}).

When $L$ is fixed, Problem (\ref{eq:objnew}) equals to 
\begin{equation}
	\min_{W}\frac{1}{2}\sum_{t=1}^T||W_t^TX_t-y_t||_2^2 + \frac{\lambda}{2}\sum_{i,j \in G}(l_{i,j}||W_i - W_j||_2^2)
	\label{eq:solveW1}
\end{equation}

In order to solve Problem (\ref{eq:solveW1}), let us define $X \in \mathbb{R}^{TP \times N}$ as a block diagonal matrix
\[
X =
\begin{bmatrix}
X_{1} & & \\
& \ddots & \\
& & X_{T}
\end{bmatrix},
\]

define $Y \in \mathbb{R}^{1\times N}$ as a row vector
$$Y = [y_1,\dots,y_T],$$

and define $V \in \mathbb{R}^{TP\times 1}$ as a column vector
\[
V =
\begin{bmatrix}
W_{1} \\
\vdots \\
W_{T}
\end{bmatrix}
\]

Then, Problem (\ref{eq:solveW1}) can be rewritten as:
\begin{equation}
\min_{V}\frac{1}{2}||V^TX-y||_2^2 + \frac{\lambda}{2}\sum_{i,j \in G}(l_{i,j}||V((e_i -e_j)\otimes I_P)||_2^2)
\label{eq:solveV}
\end{equation}
where $e_i \in \mathbb{R}^{T}$ is an indicator vector with the $i$th element set to $1$ and others $0$ and $I_P$ is an identity matrix of size $P$.

Setting the derivative of Problem (\ref{eq:solveV}) with respect to $V$ to zero, the optimal solution of $V$ can be obtained by solving the following linear system:
\begin{equation}
(A+B)V = C
\label{eq:linearsys}
\end{equation} 
where 
$A=(\lambda \sum_{i,j\in G} l_{i,j}(e_i-e_j)(e_i-e_j)^T) \otimes I_{P}$, 
$B=XX^T$, $C=XY^T$,
and $W$ is derived by reshaping $V$. 

\subsection{Convergence Analysis}
%

\begin{theorem}	
	Alternately updating Eq.~(\ref{eq:lij}) and Eq.~(\ref{eq:linearsys}) on $L$ and $W$ converges to a global optimum of Problem (\ref{eq:obj}).
\end{theorem}
\begin{proof}
	Problem~(\ref{eq:objnew}) is biconvex on $W$ and $L$. Therefore, alternately updating Eq.~(\ref{eq:lij}) and Eq.~(\ref{eq:linearsys}) will converge to a local minimum \cite{beck2015convergence}. Here we prove it in a different way. We show that this method is a contracting map $M$, so that $d(M(x),M(y)) \le d(x,y)$. Let $J(W,L)$ equals to the objective function defined in Problem (\ref{eq:objnew}). We use $d(x,y) = |J(W,L)-J(W',L')|$, with $x=(W,L),y=(W',L')$. We then demonstrate that 
	\begin{equation*}
	|J(W^+,L^+)-J(W^*,L^*)| \le |J(W^-,L^-)-J(W^*,L^*)|
	\end{equation*}
	where\footnote{the superscript *,+ and - indicate the optimal, the starting and the next update values of the variables $W$ and $L$.} $(W^*,L^*)=M(W^*,L^*)$ because $(W^*,L^*)$ is a stationary point. We define $(W^+,L^+) = M(W^-,L^-)$, where $(W^-,L^-),(W^+,L^+)$ are the variables before and after the mapping. Our mapping $M$ is composed of two steps $W^- \to W^+$ followed by $L^- \to L^+$. We show that the first step is a contraction and that the same is true also
    for the second step, 
and therefore
    for their composition. The first step updates the $W^-$ using the gradient of $\nabla_W J(W,L^-) = 0$ and finds the optimal minimum having fixed $L^-$, thus
\begin{equation*}
|J(W^+,L^-)-J(W^*,L^*)| \le |J(W^-,L^-)-J(W^*,L^*)|
\end{equation*}
The second step updates $l_{ij}$ such that $\nabla_L J(W^+,L) = 0$, thus
\begin{equation*}
|J(W^+,L^+)-J(W^*,L^*)| \le |J(W^+,L^-)-J(W^*,L^*)|,
\end{equation*}
where the reduction is both obtained by a gradient descent step or by direct solution since $J$ is convex in the two variables separately. By applying composition of the two operations we have 
\begin{eqnarray*}
|J(W^+,L^+)-J(W^*,L^*)| &\le |J(W^+,L^-)-J(W^*,L^*)|\\ &\le |J(W^-,L^-)-J(W^*,L^*)|
\end{eqnarray*}
which proves the convergence of the method. 
	
Suppose the method converges to a local optimal solution $(W^*,L^*)$ of Problem (\ref{eq:objnew}). Since $L^*$ satisfies Eq. (\ref{eq:lij}), $W^*$ is also optimal for Problem (\ref{eq:obj}) according to Theorem \ref{th:lij}. Problem (\ref{eq:obj}) is convex, thus $W^*$ is a global optimal solution of Problem (\ref{eq:obj}). 
\end{proof}

\subsection{Efficient and Scalable Implementation}
\begin{algorithm}[!t]
	\SetNoFillComment
	\SetKwInOut{Input}{Input}
	\SetKwInOut{Output}{Output}
	\Input{$\{X_t, Y_t\}$ for $t=\{1, 2, ..., T\}$, $\lambda$}
	\Output{$W = \{W_1, ..., W_T\}$}
	\For{$t\gets1$ \KwTo $T$}{
		Solve $W_t$ by Linear Regression on $\{X_t, Y_t\}$
	}
	Construct $k$-nearest neighbor graph $G$ on $W$\; 
	\While{not converge}{
		Update $L$ using Eq. (\ref{eq:lij})\;
		Update $W$ by solving Eq. (\ref{eq:linearsys}) using CMG \cite{KoutisMT11}\;
	}
	
	\Return $W$\;
	\caption{CCMTL}
	\label{alg:ccmtl}
\end{algorithm}
CCMTL is summarized in Algorithm \ref{alg:ccmtl}. Firstly, it initializes the weight vector $W_t$ by performing Linear Regression on each task $t$ separately. Then, it constructs the $k$-nearest neighbor graph $G$ on $W$ based on the Euclidean distance. Finally, the optimization problem is solved by iteratively updating $L$ and $W$ until convergence.

Solving the linear system in Eq. (\ref{eq:linearsys}) involves inverting a matrix of size $PT\times PT$, where $P$ is the number of features and $T$ is the number of tasks. Direct inversion will lead to cubic computational complexity, which will not scale to a large number of tasks. However, there are certain properties of $A$ and $B$ in Eq. (\ref{eq:linearsys}) that can be used to derive efficient implementation.

\begin{theorem}
	$A$ and $B$ in Eq. (\ref{eq:linearsys}) are both Symmetric and Positive Semi-definite and $A$ is a Laplacian matrix.
	\label{tm:psd}
\end{theorem}
\begin{proof}
	Clearly, $A$ and $B$ are Symmetric. For each pair of edge $(i,j)$ in $G$, $(e_i-e_j)(e_i-e_j)^T$ is a Laplacian matrix by definition. Since the summation of Laplacian matrices and multiplying a positive value to a Laplacian matrix are still Laplacian, $\lambda \sum_{i,j\in G} l_{i,j}(e_i-e_j)(e_i-e_j)^T$ is a Laplacian matrix since $\lambda \ge 0$ and $l_{ij}\ge 0$. The Kronecker product of a Laplacian matrix and an identity matrix will lead to a block diagonal matrix, where the diagonals are all Laplacian matrices. Therefore, $A$ is a Laplacian matrix and it is Positive Semi-definite. $B$ is a dot product, therefore it is Positive Semi-definite as well.
\end{proof}

Based on Theorem \ref{tm:psd}, Eq. (\ref{eq:linearsys}) can be solved efficiently by Conjugate Gradient (CG) method, which requires the input matrix to be Symmetric and Positive Semi-definite. CG will be faster than direct inversion since $A$ and $B$ are sparse matrices. In addition, the solution of $W$ of the previous iteration can be used to initialize CG for the new iteration, which will increase the convergence speed of CG.

Further, recent studies \cite{cohen2014solving,kelner2013simple} show that linear systems with sparse Laplacian matrices can be solved in near-linear time. Among them, \cite{KoutisMT11} proposed the Combinatorial MultiGrid (CMG) algorithm, which is also a CG based method that utilizes the hierarchal structure of the Laplacian matrix. CMG empirically scales linearly 
w.r.t. 
the non-zero entries in the Laplacian matrix of the linear system. 

Based on Theorem \ref{tm:psd}, $A$ is a sparse Laplacian matrix. Therefore, we adopt CMG to solve Eq. (\ref{eq:linearsys}) in CCMTL. Although, $A+B$ is not a Laplacian matrix anymore, empirically we find out that CMG still converges fast and scales linearly in the number of tasks. 
We conjecture that this is due to the fact that $B$ is a structured block diagonal matrix.



The runtime of CCMTL consists of threefold: 1) Initialization of $W$, 2) $k$-nearest neighbor graph construction, and 3) optimization of Problem (\ref{eq:objnew}). Clearly, initialization of $W$ by linear regression is efficient and scales linearly on the number of tasks $T$. $k$-nearest neighbor graph construction naively scales quadratically to $T$. However, it is not the burden when using MATLAB's \emph{pdist2} function even for a synthetic dataset with $160,000$ tasks. 
CCMTL empirically converges fast within around $30$ iterations. The majority of runtime is for
solving the linear system in Eq. (\ref{eq:linearsys}). The adopted CMG method scales empirically linearly in the number of non-zero entries in $A$ and $B$ in Eq. (\ref{eq:linearsys}), which is linear to $T$.

\section{Experiments}
\subsection{Comparison Methods}
We compare our method \textbf{CCMTL} with several SoA methods. As baselines, we compare with Single-task learning (\textbf{STL}), which learns a single model by pooling together the data from all the tasks and Independent task learning (\textbf{ITL}), which learns each task independently. These baselines represent the two extreme hypothesis, full independence of the tasks (ITL) and complete correlation of all tasks (STL). MTL methods should find the right balance between grouping tasks and isolating groups to achieve learning generalization.
We further compare to multi-task feature learning method: Joint Feature Learning (\textbf{L21}) \cite{argyriou2007multi} and low-rank methods: Trace-norm Regularized Learning (\textbf{Trace}) \cite{ji2009accelerated} and Robust Multi-task Learning (\textbf{RMTL}) \cite{chen2011integrating}.
We also compare to the other five clustering approaches: \textbf{CMTL} \cite{jacob2009clustered}, \textbf{FuseMTL} \cite{zhou2012modeling}, \textbf{SRMTL} \cite{zhou2011malsar} and two recently proposed models \textbf{BiFactorMTL} and \textbf{TriFactorMTL} \cite{murugesan2017co}.

All methods are implemented in Matlab and evaluated on a single thread. We implement CCMTL\footnote{ccmtlaaai.neclab.eu}, STL and ITL. We use the implementation of L21, Trace, RMTL, SRMTL and FuseMTL from the Malsar package \cite{zhou2011malsar}. We get the CMTL, BiFactorMTL and TriFactorMTL from the authors' personal website. The number of nearest neighbors $k$ is set to $10$ to get the initial graph. We use the same graph for FuseMTL, SRMTL and CCMTL generated as in Algorithm \ref{alg:ccmtl}.  CCMTL, STL, ITL, L21, Trace and FuseMTL need one hyperparameter that is selected from $[10^{-5},10^{5}]$. RMTL, CMTL, BiFactorMTL and TriFactorMTL needs two hyperparameters that are selected from $[10^{-3},10^{3}]$. CMTL and BiFactorMTL further need one and TriFactor further needs two hyperparameters for the number of clusters that are chosen from $[2, 3, 5, 10, 20, 30, 50]$. All these hyperparameters are selected by internal $5$-fold cross validation grid search on the training data.

\begin{table}
	\centering
	\caption{Summary statistic of the datasets}
	\begin{tabular}{c|ccc}
		\hline Name & Samples & Features & Num Tasks \\ \hline
		Syn & 3000 & 15 & 30 \\ 
        ScaleSyn & [500k,16000k] & 10 & [50k,160k] \\ \hline
		School & 15362 & 28 & 139 \\ \hline
		Sales & 34062 & 5 & 811 \\
		Ta-Feng & 2619320 & 5 & 23812 \\ \hline
		Alighting & 33945 & 5 & 1926 \\
		Boarding & 33945 & 5 & 1926 \\ \hline
	\end{tabular}
	\label{tb:data}
\end{table}

\begin{table*}	
	\centering
	\caption{Objective and runtime comparison between the proposed and the ADMM solver on \texttt{Syn} data}
	\begin{tabular}{|c|cc|cc|cc|cc|cc|}
		\hline
		\multirow{2}{*}{Syn} &
		\multicolumn{2}{c|}{$\lambda=0.01$} &
		\multicolumn{2}{c|}{$\lambda=0.1$} &
		\multicolumn{2}{c|}{$\lambda=1$} &
		\multicolumn{2}{c|}{$\lambda=10$} &
		\multicolumn{2}{c|}{$\lambda=100$} \\
		& Obj & Time(s) & Obj & Time(s) & Obj & Time(s) & Obj & Time(s) & Obj & Time(s)\\
		\hline
		ADMM & \textbf{1314} & 8 & \textbf{1329} & 8 & 1474 & 9 & 2320 & 49 & 7055 & 180\\
		\hline
		The proposed & \textbf{1314} & \textbf{0.5} & \textbf{1329} & \textbf{0.5} & \textbf{1472} & \textbf{0.5} & \textbf{2320} & \textbf{0.5} & \textbf{6454} & \textbf{0.5}\\
		\hline
	\end{tabular}
	\label{tb:admmsyn}
\end{table*}

\begin{table*}
	\centering
	\caption{Objective and runtime comparison between the proposed and the ADMM solver on \texttt{School} data}
	\begin{tabular}{|c|cc|cc|cc|cc|cc|}
		\hline
		\multirow{2}{*}{School} &
		\multicolumn{2}{c|}{$\lambda=0.01$} &
		\multicolumn{2}{c|}{$\lambda=0.1$} &
		\multicolumn{2}{c|}{$\lambda=1$} &
		\multicolumn{2}{c|}{$\lambda=10$} &
		\multicolumn{2}{c|}{$\lambda=100$} \\
		& Obj & Time(s) & Obj & Time(s) & Obj & Time(s) & Obj & Time(s) & Obj & Time(s)\\
		\hline
		ADMM & 664653 & 605 & 665611 & 583 & 674374 & 780 & 726016 & 4446 & 776236 & 5760\\
		\hline
		The proposed & \textbf{664642} & \textbf{0.7} & \textbf{665572} & \textbf{0.8} & \textbf{674229} & \textbf{0.9} & \textbf{725027} & \textbf{1.5} & \textbf{764844} & \textbf{1.9}\\
		\hline
	\end{tabular}
	\label{tb:admmschool}
\end{table*}

\begin{table}
	\centering
	\caption{Results (RMSE) on \texttt{Syn} dataset. The table reports the mean and standard errors over $5$ random runs. The best model and the statistical competitive models (by paired \emph{t-test} with $\alpha=0.05$) are shown in bold. }
	\begin{tabular}{|c|ccc|}
		\hline & 20$\%$ & 30$\%$ & 40$\%$ \\ \hline
		STL & 2.905 \scriptsize{(0.031)} & 2.877 \scriptsize{(0.025)} & 2.873 \scriptsize{(0.036)} \\
		ITL & 1.732 \scriptsize{(0.077)} & 1.424 \scriptsize{(0.049)} & 1.284 \scriptsize{(0.024)} \\ \hline
		L21 & 1.702 \scriptsize{(0.033)} & 1.388 \scriptsize{(0.014)} & 1.282 \scriptsize{(0.011)} \\
		Trace & 1.302 \scriptsize{(0.042)} & 1.222 \scriptsize{(0.028)} & 1.168 \scriptsize{(0.023)} \\ 
		RMTL & 1.407 \scriptsize{(0.028)} & 1.295 \scriptsize{(0.024)} & 1.234 \scriptsize{(0.039)} \\ \hline
		CMTL & 1.263 \scriptsize{(0.038)} & 1.184 \scriptsize{(0.007)} & 1.152 \scriptsize{(0.017)} \\ 
		FuseMTL & 2.264 \scriptsize{(0.351)} & 1.466 \scriptsize{(0.025)} & 1.297 \scriptsize{(0.048)} \\
        SRMTL & 1.362 \scriptsize{(0.018)} & 1.195 \scriptsize{(0.014)} & 1.152 \scriptsize{(0.012)}  \\        
		BiFactor & \textbf{1.219 \scriptsize{(0.025)}} & \textbf{1.150 \scriptsize{(0.020)}} & \textbf{1.125 \scriptsize{(0.013)}} \\
		TriFactor & 1.331 \scriptsize{(0.239)} & 1.255 \scriptsize{(0.236)} & \textbf{1.126 \scriptsize{(0.010)}}      \\ \hline
		CCMTL & \textbf{1.192 \scriptsize{(0.018)}} & \textbf{1.161 \scriptsize{(0.018)}} & \textbf{1.136 \scriptsize{(0.015)}} \\ \hline
	\end{tabular}
	\label{tb:syn}
\end{table}

\subsection{Datasets}
We employ both synthetic and real-world datasets. Table \ref{tb:data} shows their statistics. Further details are provided below.
\subsubsection{Accuracy Synthetic.} \texttt{Syn} dataset aims at showing the ability of MTL methods to capture tasks structure. It consists of $3$ groups of tasks with $10$ tasks in each group. We generate $15$ features $F$ from $\mathcal{N}(0,1)$. Tasks in group $1$  are constructed from features $1-5$ in $F$ and random $10$ features. Similarly, Tasks in group $2$ and $3$ are constructed from features $6-10$ and $11-15$ in $F$ respectively. $100$ samples are generated for each task.

\subsubsection{Scaling Synthetic.} \texttt{ScaleSyn} datasets aim at showing the computational performance of MTL methods. It has fixed feature size (i.e. $P$=10), but an exponentially growing number of tasks (from $5k$ to $160k$). Tasks are generated in groups of fixed size ($100$). The latent features for tasks in the same group are sampled from $\mathcal{N}(C,1)$, where the center $C$ is sampled from $\mathcal{N}(0,1000)$. $100$ samples are generated for each task as well.

\subsubsection{Exam Score Prediction.} \texttt{School} is a classical benchmark dataset in Multi-task regression reported in literatures \cite{argyriou2007multi}, \cite{kumar2012learning}, \cite{zhang2014regularization}. It consists of examination scores of $15,362$ students from $139$ schools in London. Each school is considered as a task and the aim is to predict the exam scores for all the students. We use the dataset from Malsar package \cite{zhou2011malsar}.

\subsubsection{Retail.} \texttt{Sales} is a dataset contains weekly purchased quantities of $811$ products over $52$ weeks \cite{tan2014time}. We acquired the dataset from UCI repository \cite{Dua:2017}. We build the dataset by using the sales quantities of $5$ previous weeks for each product to predict the sales for the current week, resulting in $34,062$ samples in total. \texttt{Ta-Feng} is another grocery shopping large dataset that consists of transactions data of $23,812$ products over $4$ months. We build the data in a similar fashion obtaining $2,619,320$ samples in total.

\subsubsection{Transportation.} Demand prediction is an important aspect for Intelligent Transportation Systems (ITS). We used a confidential real dataset consisting of bus arrival time and passenger counting information at each station for two lines of a major European city in both directions with four trip each. A task (total of $1,926$) consists on the prediction of the passenger demand at each stop, given the arrival time to the stop and the number of alighting and boarding at the previous two stops. The \texttt{alighting} and \texttt{boarding} datasets contain $33,945$ samples and $5$ features.

\begin{table}
	\centering
	\caption{Results (RMSE) on \texttt{School} dataset. The table reports the mean and standard errors over $5$ random runs. The best model and the statistical competitive models (by paired \emph{t-test} with $\alpha=0.05$) are shown in bold.}
	\begin{tabular}{|c|ccc|}
		\hline & 20$\%$ & 30$\%$ & 40$\%$ \\ \hline
		STL & 10.245 \scriptsize{(0.026)} & 10.219 \scriptsize{(0.034)} & 10.241 \scriptsize{(0.068)} \\ 
		ITL & 11.427 \scriptsize{(0.149)} & 10.925 \scriptsize{(0.085)} & 10.683 \scriptsize{(0.045)} \\ \hline
		L21 & 11.175 \scriptsize{(0.079)}  & 11.804 \scriptsize{(0.134)} & 11.442 \scriptsize{(0.137)} \\ 
		Trace & 11.117 \scriptsize{(0.054)} & 11.877 \scriptsize{(0.542)} & 11.655 \scriptsize{(0.058)} \\ 
		RMTL & 11.095 \scriptsize{(0.066)} & 10.764 \scriptsize{(0.068)} & 10.544 \scriptsize{(0.061)} \\ \hline
		CMTL & \textbf{10.219 \scriptsize{(0.056)}}&  10.109 \scriptsize{(0.069)}&  10.116 \scriptsize{(0.053)}  \\ 
		FuseMTL & 10.372 \scriptsize{(0.108)} & 10.407 \scriptsize{(0.269)} & 10.217 \scriptsize{(0.085)} \\
        SRMTL & 10.258 \scriptsize{(0.022)} & 10.212 \scriptsize{(0.039)} & 10.128 \scriptsize{(0.021)} \\        
		BiFactor & 10.445 \scriptsize{(0.135)} & 10.201 \scriptsize{(0.067)} & 10.116 \scriptsize{(0.051)} \\ 
		TriFactor & 10.551 \scriptsize{(0.080)} & 10.224 \scriptsize{(0.070)} & 10.129 \scriptsize{(0.020)} \\ \hline
		CCMTL & \textbf{10.170 \scriptsize{(0.029)}} & \textbf{10.036 \scriptsize{(0.046)}} & \textbf{10.020 \scriptsize{(0.021)}} \\ \hline
	\end{tabular}
	\label{tb:school}
\end{table}

\subsection{Results and Discussion}
%

\subsubsection{Comparison with ADMM-based Solver.}
Firstly, we compare our solver with an ADMM-based solver when determining a solution to our problem of (\ref{eq:obj}). The ADMM-based solver is implemented using SnapVX python package from NetworkLasso \cite{hallac2015network}. Both solvers are evaluated on the \texttt{Syn} and the \texttt{School} benchmark datasets. The $k$-nearest graph, $G$, is generated as described in Algorithm \ref{alg:ccmtl} and is used to test both solvers with different values for the regularization parameter $\lambda$.
Tables \ref{tb:admmsyn} and \ref{tb:admmschool} show the final objective functions and runtime comparison on \texttt{Syn} and \texttt{School} datasets, respectively. It is clear that, for small $\lambda$ ( $\leq 1$), both solver achieve similar objective values for the problem (\ref{eq:obj}). When $\lambda$ takes larger values ($>1$), the objective values of ADMM method tend to monotonically increase, reflecting the increasing importance of the regularization term, but with a smaller slope for our solvers compared to the ADMM-based one. 
In addition to the lower objective function, our solver is clearly more computationally efficient than the expensive ADMM-based solver. The proposed solver shows stability in runtime, by taking at maximum two seconds for all possible $\lambda$ values, compared to a runtime in the range of $[605-5760]$ seconds for the ADMM-based solver, on the \texttt{School} data.

%

\subsubsection{Comparison with SoA MTL methods.}
CCMTL is compared with the state-of-the-art methods in terms of the Root Mean Squared Error (RMSE). All experiments are repeated $5$ times with different shuffling. 
In all result's tables, we compare the best performing method with the remaining ones using the paired \emph{t-test} (with $\alpha=0.05$). The best method and the methods that cannot be statistically outperformed (by the best one) are shown in boldface.

Table \ref{tb:syn} presents the prediction error, RMSE, on the \texttt{Syn} dataset with the ratio of training samples ranging from $20\%$ to $40\%$. Tasks in the \texttt{Syn} dataset are generated to be heterogeneous and well partitioned, therefore, STL performs the worst, since it trains only a single model on all tasks. Similarly, the baseline ITL is also outperformed by the remaining MTL methods.
Our approach, CCMTL, is statistically better than all SoA methods, except for the BiFactor which performs as well as CCMTL on the \texttt{Syn} dataset.


The results on the \texttt{School} dataset are depicted in Table \ref{tb:school} with a ratio of training samples ranging from $20\%$ to $40\%$. It appears to be that, unlike the \texttt{Syn} data, the \texttt{School} data has tasks that are rather homogeneous, therefore, ITL performs the worst and STL shows its superiority on many of the MTL methods (L21, Trace, RMTL, FuseMTL and SRMTL). MTFactor and TriFactor outperform STL only when the training ratio is larger than $30\%$ and $40\%$, respectively.
CCMTL, again, performs better than all competitive methods, on all training rations; CCMTL is also statistically the best performing method, expect for CMTL (with ratio $20\%$) where the null hypothesis could not be rejected.


\begin{table}	
	\centering
	\caption{Results (RMSE and runtime) on Retail datasets. The table reports the mean and standard errors over $5$ random runs. The best model and the statistical competitive models (by paired \emph{t-test} with $\alpha=0.05$) are shown in bold. The best runtime for MTL methods is shown in boldface.}
	\begin{tabular}{|c|cc|cc|}
		\hline
		\multirow{2}{*}{} &
		\multicolumn{2}{c|}{Sales} &
		\multicolumn{2}{c|}{Ta-Feng}  \\
		& RMSE & Time(s) & RMSE & Time(s) \\
		\hline
		STL & 2.861 \scriptsize{(0.02)} & 0.1 & 0.791 \scriptsize{(0.01)} & 0.2 \\ 
		ITL & 3.115 \scriptsize{(0.02)} & 0.1 & 0.818 \scriptsize{(0.01)} & 0.4 \\ \hline
		L21 & 3.301 \scriptsize{(0.01)} & 11.8 & 0.863 \scriptsize{(0.01)} & 831.2 \\ 
		Trace & 3.285 \scriptsize{(0.21)} & 10.4 & 0.863 \scriptsize{(0.01)} & 582.3 \\ 
		RMTL & 3.111 \scriptsize{(0.01)} & 3.4 & 0.833 \scriptsize{(0.01)} & 181.5\\ \hline
		CMTL & 3.088 \scriptsize{(0.01)} & 43.4 & - & $> 24h$ \\ 
		FuseMTL & 2.898 \scriptsize{(0.01)} & 4.3 & \textbf{0.764 \scriptsize{(0.01)}} & 8483.3 \\ 
        SRMTL & 2.854 \scriptsize{(0.02)} & 10.3 & - & $> 24h$ \\        
		BiFactor & 2.882 \scriptsize{(0.01)} & 55.7 & - & $> 24h$ \\ 
		TriFactor & 2.857 \scriptsize{(0.04)} & 499.1 & - & $> 24h$ \\ \hline
		CCMTL & \textbf{2.793 \scriptsize{(0.01)}} & \textbf{1.8} & \textbf{0.767 \scriptsize{(0.01)}} & \textbf{35.3} \\ \hline
	\end{tabular}
	\label{tb:retail}
\end{table}

Table \ref{tb:retail} depicts the results on two retail datasets: \texttt{Sales} and \texttt{Ta-Feng}; it also depicts the time required (in seconds) for the training using the best found parametrization for each method. Here $50\%$ of samples are used for training. The best runtime for MTL methods is shown in boldface. Tasks in these two datasets are, again, rather homogeneous, therefore, the baseline STL has a competitive performance and outperforms many MTL methods. STL outperforms ITL, L21, Trace\footnote{The hyperparameters searching range for L21 and Trace are shifted to $[10^0, 10^{10}]$ for \texttt{Ta-Feng} dataset to get reasonable results.}, RMTL, CMTL, FuseMTL, SRMTL, BiFactor, and TriFactor on the \texttt{Sales} dataset, and outperforms ITL, L21, Trace and RMTL on the \texttt{Ta-Feng} data\footnote{We set a timeout at $24h$. CMTL, SRMTL, BiFactor, and TriFactor did not return the result on this timeout for the \texttt{Ta-Feng} dataset.}.  
CCMTL is the only method that performs better (also statistically better) than STL on both data sets; it also outperforms all MTL methods (with statistical significance) on both data sets, except for FuseMTL which performs slightly better than CCMTL only on the \texttt{Ta-Feng} data. CCMTL requires the smallest runtime in comparison with the competitor MTL algorithms. On \texttt{Ta-Feng} dataset, CCMTL requires only around $30$ seconds, while the SoA methods need up to hours or even days.

\begin{table}	
	\centering
	\caption{Results (RMSE and runtime) on Transportation datasets. The table reports the mean and standard errors over $5$ random runs. The best model and the statistical competitive models (by paired \emph{t-test} with $\alpha=0.05$) are shown in bold. The best runtime for MTL methods is shown in boldface.}
	\begin{tabular}{|c|cc|cc|}
		\hline
		\multirow{2}{*}{} &
		\multicolumn{2}{c|}{Alighting} &
		\multicolumn{2}{c|}{Boarding}  \\
		& RMSE & Time(s) & RMSE & Time(s) \\
		\hline
		STL & 3.073 \scriptsize{(0.02)} & 0.1 & 3.236 \scriptsize{(0.03)} & 0.1\\ 
		ITL & 2.894 \scriptsize{(0.02)} & 0.1 & 3.002 \scriptsize{(0.03)} & 0.1 \\ \hline
		L21 & 2.865 \scriptsize{(0.04)} & 14.6 & 2.983 \scriptsize{(0.03)} & 16.7 \\ 
		Trace & 2.835 \scriptsize{(0.01)} & 19.1 & 2.997 \scriptsize{(0.05)} & 17.5 \\ 
		RMTL & 2.985 \scriptsize{(0.03)} & 6.7 & 3.156 \scriptsize{(0.04)} & 7.1\\ \hline
		CMTL & 2.970 \scriptsize{(0.02)} & 82.6 & 3.105 \scriptsize{(0.03)} & 91.8\\ 
		FuseMTL & 3.080 \scriptsize{(0.02)} & 11.1 & 3.243 \scriptsize{(0.03)} & 11.3\\ 
        SRMTL & \textbf{2.793 \scriptsize{(0.02)}} & 12.3 & \textbf{2.926 \scriptsize{(0.02)}} & 14.2 \\        
		BiFactor & 3.010 \scriptsize{(0.02)} & 152.1 & 3.133 \scriptsize{(0.03)} & 99.7\\
		TriFactor & 2.913 \scriptsize{(0.02)} & 282.3 & 3.014 \scriptsize{(0.03)} & 359.1 \\ \hline
		CCMTL & \textbf{2.795 \scriptsize{(0.02)}} & \textbf{4.8} & \textbf{2.928 \scriptsize{(0.03)}} & \textbf{4.1} \\ \hline
	\end{tabular}
	\label{tb:transport}
\end{table}

Table \ref{tb:transport} depicts the results on the \texttt{Transportation} datasets, using two different target attributes (alighting and boarding); again, the runtime is presented for the best-found parametrization, and the best runtime achieved by the MTL methods are shown in boldface. The results on this dataset are interesting, especially because both baselines are not competitive as in the previous datasets. This could, safely, lead to the conclusion that the tasks belong to latent groups, where tasks are homogeneous intra-group, and heterogeneous inter-groups. All MTL methods (except the FuseMTL) outperform at least one of the baselines (STL and ITL) on both datasets. 
Our approach, CCMTL, seems to reach the right balance between task independence (ITL) and complete correlation (STL), as confirmed by the results; it achieves, statistically, the lowest RMSE against the baselines and the all other MTL methods (except SRMTL), and it is at least 40\% faster than the fastest MTL method (RMTL).

\begin{figure*}
	\centering
	\includegraphics[width=0.85\textwidth]{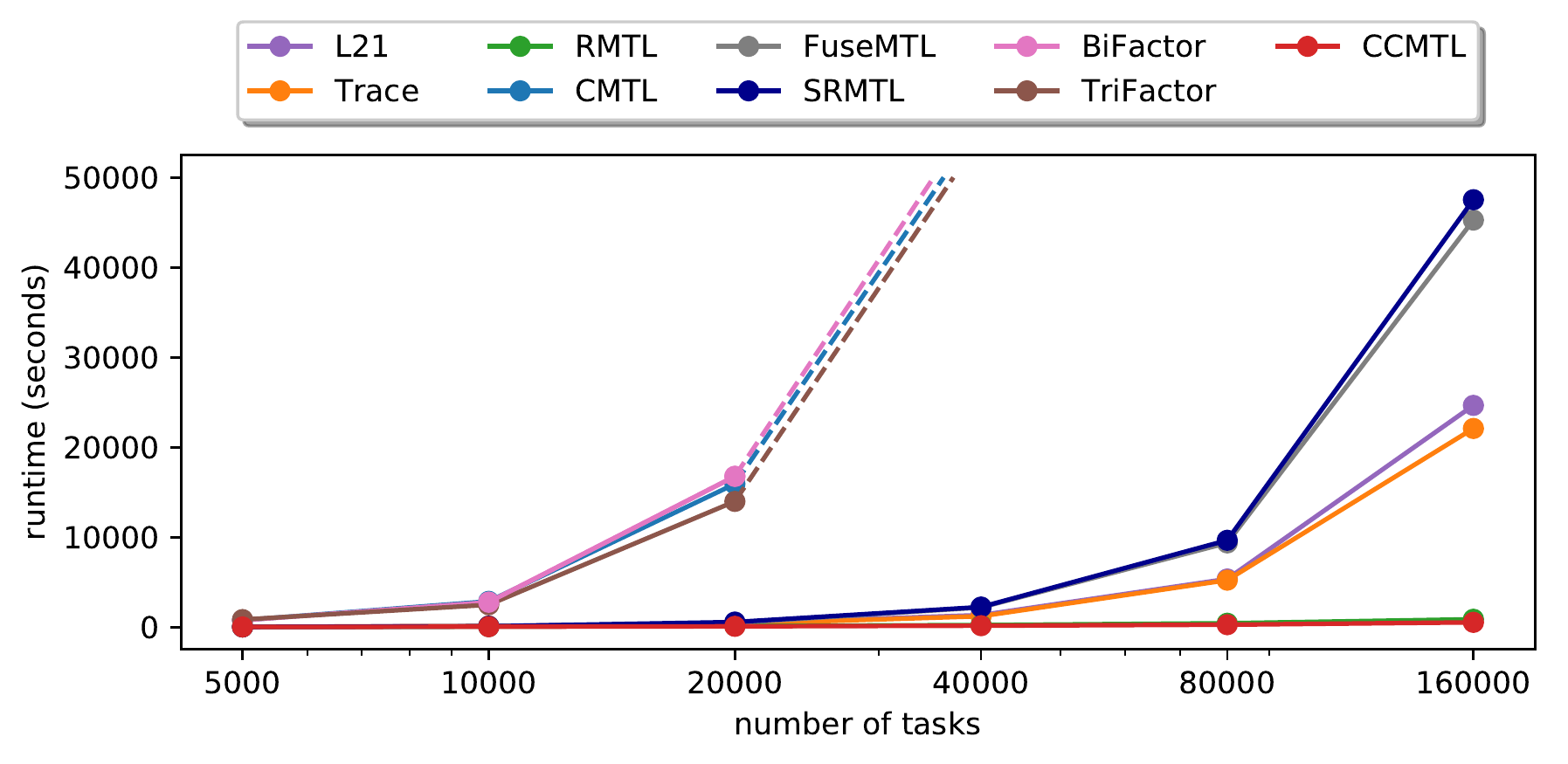}
	\caption{Scalability experiments on synthetic datasets with increasing number of tasks (log scale on task number)}
	\label{fig:scale}
\end{figure*}

\subsubsection{Scalability.}
In the scalability analysis, we use the \texttt{ScaleSyn} dataset. We search the hyperparameters for all the methods on the smallest one with $5k$ tasks and evaluate the runtime of the best-found hyperparameters on all the others.
Figure \ref{fig:scale} shows the recorded runtime in seconds while presenting the number of tasks in the log-scale. As can be seen, CMTL, MTFactor, and TriFactor were not capable to process 40k tasks in less than 24 hours, therefore, they were stopped (the extrapolation for the minimum needed runtime can be seen as a dashed line). FuseMTL, SRMTL, L21, and Trace tend to show a super-linear growth of the needed runtime in the log-scale. Both CCMTL and RMTL show constant behavior in the number of tasks, where only around $500$ and $800$ seconds are needed for the dataset with $160k$ tasks respectively. CCMTL is the fastest method among all the MTL competitors. A regression analysis on the runtime curves is presented in the supplementary material; this analysis shows that only CCMTL and RMTL scales linearly, compared to the other methods that scale quadratically.

\section{Related Work}
Zhang and Yang \cite{Zhang2017Survey}, in their survey, classify multi-task learning into different categories: feature learning, low-rank approaches, and task clustering approaches, among others. These categories are characterized mainly by how the information between the different tasks is shared and which information is subject to sharing.

One type of MTL perform joint feature learning (L21) that assumes all tasks share a common set of features and penalizes it by $\ell_{2,1}$-norm regularization \cite{argyriou2007multi,argyriou2008convex,liu2009multi}. Another way to capture the task relationship is to constrain the models from different tasks to share a low-dimensional subspace, i.e. $W$ is of low-rank (Trace) \cite{ji2009accelerated}. Both L21 and Trace assumes all the tasks are relevant, which is usually not true in real-world applications. Chen et al. \cite{chen2011integrating} propose robust multi-task learning (RMTL) in identifying irrelevant tasks by integrating the low-rank and group-sparse structures. These methods are relatively fast but cannot 
capture the task relationship when they belong to different latent groups.

The task clustering approaches aim at solving this issue where different tasks form clusters of similar tasks. 
Jacob et al. \cite{jacob2009clustered} propose to integrate the objective of $k$-means into the learning framework and solve a relaxed convex problem. Zhou et al. \cite{zhou2011clustered} use a similar idea for task clustering, but with a different optimization method. Zhou and Zhao \cite{zhou2016flexible} propose to cluster tasks by identifying representative tasks. Another way of performing task clustering is through the decomposition of the weight matrix $W$ \cite{kumar2012learning,barzilai2015convex}. Later, a similar idea is performed with co-clustering of the features and the tasks \cite{murugesan2017co}. Despite being effective, these methods are expensive to train and the number of clusters is needed as a hyperparameter which makes the model tuning even more difficult.

Fused Multi-task Learning (FuseMTL) \cite{zhou2012modeling,chen2010graph} and Graph regularized Multi-task Learning (SRMTL) \cite{zhou2011malsar} are the most related works, where $\ell_1$ norm and squared $\ell_2$ is used as the regularizer. As shown in the experiments, with $\ell_2$ norm CCMTL outperforms FuseMTL and SRMTL in most cases. In addition, the underlying optimization method is also different, where the proposed CCMTL runs much faster. Another closely related work is a multi-level clustering method \cite{han2015learning}, where objective function also uses $l_2$ norm. It looks similar to the proposed one if only one layer is considered \cite{Zhang2017Survey}. However, the method scales quadratically since constraint is on all the pairs of tasks, making it unsuitable for the studied problem in this paper with a massive number of tasks.

\section{Conclusion}
In this paper, we study the multi-task learning problem with a massive number of tasks. We integrate \emph{convex clustering} into the multi-task regression learning problem that captures tasks' relationships on the $k$-NN graph of the prediction models. Further, we present an approach CCMTL that solves this problem efficiently and is guaranteed to converge to the global optimum. Extensive experiments show that CCMTL makes a more accurate prediction, runs faster than SoA competitors on both synthetic and real-world datasets, and scales linearly in the number of tasks. CCMTL will serve as a method for a wide range of large-scale regression applications where the number of tasks is tremendous. In the future, we will explore the use of the proposed method for online learning and high-dimensional problem.

\section{ Acknowledgments}
The authors would like to thank Luca Franceschi for the discussion of the paper and the anonymous reviewers for their helpful comments and Prof. Yiannis Koutis for sharing the implementation of CMG.

\bibliography{bib}
\bibliographystyle{aaai}

\end{document}